\pdfoutput=1
\documentclass[11pt]{article}

\usepackage[dvipsnames]{xcolor}

\usepackage{naacl2021}

\usepackage{times}
\usepackage{latexsym}

\usepackage[T1]{fontenc}
\usepackage[utf8]{inputenc}

\usepackage{microtype}

\usepackage{amsmath,amssymb,amsthm,bbm}
\usepackage{ulem}
\usepackage{tikz-qtree}
\usepackage{endnotes}
\usepackage{graphicx}
\usepackage{endnotes}
\usepackage{hyperref}
\usepackage{fancyhdr}
\usepackage{datetime}
\usepackage{forest}
\usepackage{mdframed}
\usepackage{xspace}

\usepackage{tikz}
\def\checkmark{\tikz\fill[scale=0.4](0,.35) -- (.25,0) -- (1,.7) -- (.25,.15) -- cycle;}

\newcommand{\pop}{\ensuremath{\mathbf{pop}}}

\newcommand{\doubbr}[1]{\left[\left[#1\right]\right]}
\newcommand{\singbr}[1]{\left[#1\right]}
\newcommand{\myset}[1]{\left\{#1\right\}}
\newcommand{\paren}[1]{\left(#1\right)}
\newcommand{\myexp}[1]{\exp\paren{#1}}

\newcommand{\expected}[2]{\underset{#1}{\E}\singbr{#2}}

\newcommand{\abs}[1]{\left|#1\right|}

\newcommand{\Prob}[2]{\ensuremath{\underset{#1}{\Pr}\paren{#2}}}

\newcommand{\wh}[1]{\ensuremath{\widehat{#1}}}

\DeclareFontFamily{U}{mathx}{\hyphenchar\font45}
\DeclareFontShape{U}{mathx}{m}{n}{<-> mathx10}{}
\DeclareSymbolFont{mathx}{U}{mathx}{m}{n}
\DeclareMathAccent{\widebar}{0}{mathx}{"73}
\newcommand*\wb[1]{\ensuremath{\widebar{#1}}}

\newcommand{\ep}{\ensuremath{\epsilon}}
\newcommand{\de}{\ensuremath{\delta}}

\newcommand{\ga}{\ensuremath{\gamma}}

\newcommand{\R}{\ensuremath{\mathbb{R}}}

\newcommand{\E}{\ensuremath{\mathbf{E}}}

\newcommand{\ra}{\ensuremath{\rightarrow}}

\usepackage{mathtools}

\newcommand{\by}{\times}

\definecolor{light-gray}{gray}{0.80}

\definecolor{darkred}{rgb}{0.64, 0.0, 0.0}

\theoremstyle{definition}
\newtheorem{thm}{Theorem}[section]
\newtheorem*{thm*}{Theorem}

\DeclareMathOperator*{\argmax}{arg\,max}

\newenvironment{itemizesquish}{\begin{list}{\labelitemi}{\setlength{\itemsep}{-0.2em}\setlength{\labelwidth}{0.5em}\setlength{\leftmargin}{\labelwidth}
\addtolength{\leftmargin}{\labelsep}}}{\end{list}}

\title{Understanding Hard Negatives in Noise Contrastive Estimation}
\author{Wenzheng Zhang \and Karl Stratos \\
  Department of Computer Science \\
  Rutgers University \\
  \texttt{\{wenzheng.zhang, karl.stratos\}@rutgers.edu}%\;\;\;\;\; \texttt{stratos@cs.rutgers.edu}
}

\begin{document}
\maketitle

\begin{abstract}
  The choice of negative examples is important in noise contrastive estimation. Recent works find that hard negatives---highest-scoring incorrect examples under the model---are effective in practice, but they are used without a formal justification. We develop analytical tools to understand the role of hard negatives. Specifically, we view the contrastive loss as a biased estimator of the gradient of the cross-entropy loss, and show both theoretically and empirically that setting the negative distribution to be the model distribution results in bias reduction. We also derive a general form of the score function that unifies various architectures used in text retrieval. By combining hard negatives with appropriate score functions, we obtain strong results on the challenging task of zero-shot entity linking.
\end{abstract}

\newcommand{\bm}{\textsc{bm{\scriptsize 25}}\xspace}
\newcommand{\dual}{\textsc{dual}\xspace}
\newcommand{\poly}{\textsc{poly}\xspace}
\newcommand{\genpoly}{\textsc{genpoly}\xspace}
\newcommand{\me}{\textsc{multi}\xspace}
\newcommand{\som}{\textsc{som}\xspace}
\newcommand{\bert}{\textsc{bert}\xspace}
\newcommand{\blink}{\textsc{blink}\xspace}
\newcommand{\joint}{\textsc{joint}\xspace}

\section{Introduction}

Noise contrastive estimation (NCE) is a widely used approach to large-scale classification and retrieval.
It estimates a score function of input-label pairs by a sampled softmax objective:
given a correct pair $(x, y_1)$, choose negative examples $y_2 \ldots y_K$ and maximize the probability of $(x,y_1)$ in a softmax over the scores of $(x, y_1) \ldots (x, y_K)$.
NCE has been successful in many applications, including information retrieval \citep{huang2013learning},
entity linking \citep{gillick-etal-2019-learning}, and open-domain question answering \citep{karpukhin2020dense}.

It is well known that making negatives ``hard'' can be empirically beneficial.
For example, \citet{gillick-etal-2019-learning} propose a hard negative mining strategy in which
highest-scoring incorrect labels under the current model are chosen as negatives.
Some works even manually include difficult examples based on external information such as a ranking function \citep{karpukhin2020dense} or a knowledge base \citep{wikipretrainthing}.

While it is intuitive that such hard negatives help improve the final model by making the learning task more challenging,
they are often used without a formal justification.
Existing theoretical results in contrastive learning are not suitable for understanding hard negatives since
they focus on unconditional negative distributions \citep{gutmann2012noise,mnih2012fast,ma-collins-2018-noise,tian2020makes}
or consider a modified loss divergent from practice \citep{bengio2008adaptive}.

In this work, we develop analytical tools to understand the role of hard negatives.
We formalize hard-negative NCE with a realistic loss \eqref{eq:nce-hard} using a general conditional negative distribution,
and view it as a biased estimator of the gradient of the cross-entropy loss.
We give a simple analysis of the bias (Theorem~\ref{thm:bias}).
We then consider setting the negative distribution to be the model distribution, which recovers the hard negative mining strategy of \citet{gillick-etal-2019-learning},
and show that it yields an unbiased gradient estimator when the model is optimal (Theorem~\ref{thm:unbiased}).
We complement the gradient-based perspective with an adversarial formulation (Theorem~\ref{thm:adv}).

The choice of architecture to parametrize the score function is another key element in NCE.
There is a surge of interest in developing efficient cross-attentional architectures \citep{poly,colbert,luan2020sparse},
but they often address different tasks and lack direct comparisons.
We give a single algebraic form of the score function \eqref{eq:score} that subsumes and generalizes these works,
and directly compare a spectrum of architectures it induces.

We present experiments on the challenging task of zero-shot entity linking \citep{logeswaran2019zero}.
We calculate empirical estimates of the bias of the gradient estimator to verify our analysis, and systematically explore the joint space of negative examples and architectures.
We have clear practical recommendations: (i) hard negative mining always improves performance for all architectures, and
(ii) the sum-of-max encoder \citep{colbert} yields the best recall in entity retrieval.
Our final model combines the sum-of-max retriever with a \bert-based joint reranker
to achieve 67.1\% unnormalized accuracy: a 4.1\% absolute improvement over \citet{wu2020scalable}.
We also present complementary experiments on AIDA CoNLL-YAGO \citep{hoffart2011robust}
in which we finetune a Wikipedia-pretrained dual encoder with hard-negative NCE and show a 6\% absolute improvement in accuracy.

\section{Review of NCE}

Let $\mathcal{X}$ and $\mathcal{Y}$ denote input and label spaces. We assume $\abs{\mathcal{Y}} < \infty$ for simplicity.
Let $\pop$ denote a joint population distribution over $\mathcal{X} \times \mathcal{Y}$.
We define a score function $s_\theta: \mathcal{X} \times \mathcal{Y} \ra \R$ differentiable in $\theta \in \R^d$.
Given sampling access to $\pop$, we wish to estimate $\theta$ such that
the classifier $x \mapsto \argmax_{y \in \mathcal{Y}} s_\theta(x, y)$ (breaking ties arbitrarily) has the optimal expected zero-one loss.
We can reduce the problem to conditional density estimation. Given $x \in \mathcal{X}$, define
\begin{align}
p_\theta(y|x) = \frac{\myexp{s_\theta(x,y)}}{\sum_{y' \in \mathcal{Y}} \myexp{s_\theta(x,y')}} \label{eq:model}
\end{align}
for all $y \in \mathcal{Y}$. Let $\theta^*$ denote a minimizer of the cross-entropy loss:
\begin{align}
J_{\mbox{\tiny CE}}(\theta) = \expected{(x,y) \sim \pop}{ -\log p_\theta(y|x)} \label{eq:ce}
\end{align}
If the score function is sufficiently expressive, $\theta^*$ satisfies $p_{\theta^*}(y|x)=\pop(y|x)$ by the usual property of cross entropy.
This implies that $s_{\theta^*}$ can be used as an optimal classifier.

The cross-entropy loss is difficult to optimize when $\mathcal{Y}$ is large since the normalization term in \eqref{eq:model} is expensive to calculate.
In NCE, we dodge this difficulty by subsampling.
Given $x \in \mathcal{X}$ and any $K$ labels $y_{1:K} = (y_1 \ldots y_K) \in \mathcal{Y}^K$, define
\begin{align}
\pi_\theta(k|x, y_{1:K}) = \frac{\myexp{s_\theta(x,y_k)}}{ \sum_{k'=1}^K \myexp{s_\theta(x,y_{k'})} } \label{eq:pi}
\end{align}
for all $1 \leq k \leq K$.
When  $K \ll \abs{\mathcal{Y}}$, \eqref{eq:pi} is significantly cheaper to calculate than \eqref{eq:model}.
Given $K \geq 2$, we define
\begin{align}
J_{\mbox{\tiny NCE}}(\theta) = \expected{\substack{(x,y_1) \sim \pop \\ y_{2:K} \sim q^{K-1}}}{ -\log \pi_\theta(1|x, y_{1:K}) } \label{eq:nce}
\end{align}
where $y_{2:K} \in \mathcal{Y}^{K-1}$ are negative examples drawn iid from some ``noise'' distribution $q$ over $\mathcal{Y}$.
Popular choices of $q$ include the uniform distribution $q(y) = 1/\abs{\mathcal{Y}}$ and the population marginal $q(y) = \pop(y)$.

The NCE loss \eqref{eq:nce} has been studied extensively.
An optimal classifier can be extracted from a minimizer of $J_{\mbox{\tiny NCE}}$ \citep{ma-collins-2018-noise};
minimizing $J_{\mbox{\tiny NCE}}$ can be seen as maximizing a lower bound on the mutual information between $(x,y) \sim \pop$ if $q$ is the population marginal \citep{oord2018representation}.
We refer to \citet{survey} for an overview.
However, most of these results focus on unconditional negative examples and do not address hard negatives, which are clearly conditional.
We now focus on conditional negative distributions,
which are more suitable for describing hard negatives.

\section{Hard Negatives in NCE}
\label{sec:hard}

Given $K \geq 2$, we define
\begin{align}
  J_{\mbox{\tiny HARD}}(\theta) = \expected{\substack{(x,y_1) \sim \pop \\ y_{2:K} \sim h(\cdot|x,y_1)}}{ -\log \pi_\theta(1|x,y_{1:K}) } \label{eq:nce-hard}
\end{align}
where $y_{2:K} \in \mathcal{Y}^{K-1}$ are negative examples drawn from a conditional distribution $h(\cdot|x,y_1)$ given $(x,y_1) \sim \pop$.
Note that we do not assume $y_{2:K}$ are iid.
While simple, this objective captures the essence of using hard negatives in NCE,
since the negative examples can arbitrarily condition on the input and the gold (e.g., to be wrong but difficult to distinguish from the gold)
and be correlated (e.g., to avoid duplicates).

We give two interpretations of optimizing $J_{\mbox{\tiny HARD}}$.
First, we show that the gradient of $J_{\mbox{\tiny HARD}}$ is a biased estimator of the gradient of the cross-entropy loss $J_{\mbox{\tiny CE}}$.
Thus optimizing $J_{\mbox{\tiny HARD}}$ approximates optimizing $J_{\mbox{\tiny CE}}$ when we use a gradient-based method,
where the error depends on the choice of $h(\cdot|x, y_1)$.
Second, we show that the hard negative mining strategy can be recovered by considering an adversarial setting in which $h(\cdot|x,y_1)$ is learned to maximize the loss.

\subsection{Gradient Estimation}
\label{sec:grad}

We assume an arbitrary choice of $h(\cdot|x, y_1)$ and $K \geq 2$. Denote the bias at $\theta \in \R^d$ by
\begin{align*}
  b(\theta) &= \nabla J_{\mbox{\tiny CE}}(\theta) - \nabla J_{\mbox{\tiny HARD}}(\theta)
\end{align*}
To analyze the bias, the following quantity will be important.
For $x \in \mathcal{X}$ define
\begin{align}
  \ga_\theta(y|x) &= \Prob{\substack{y_1 \sim \pop(\cdot|x) \\ y_{2:K} \sim h(\cdot|x,y_1) \\ k\sim \pi_\theta(\cdot|x,y_{1:K})}}{ y_k = y} \label{eq:gamma}
\end{align}
for all $y \in \mathcal{Y}$. That is, $\ga_\theta(y|x)$ is the probability that $y$ is included as a candidate (either as the gold or a negative)
and then selected by the NCE discriminator \eqref{eq:pi}.

\begin{thm}
  For all $i = 1\ldots d$,
\begin{align*}
  b_i(\theta) = \expected{x \sim \pop}{\sum_{y \in \mathcal{Y}} \ep_\theta(y|x) \frac{\partial s_\theta(x,y)}{\partial \theta_i}}
\end{align*}
where $\ep_\theta(y|x) = p_\theta(y|x)-\ga_\theta(y|x)$.
\label{thm:bias}
\end{thm}

\begin{proof}
  Fix any $x \in \mathcal{X}$ and let $J_{\mbox{\tiny CE}}^x(\theta)$ and $J_{\mbox{\tiny HARD}}^x(\theta)$
  denote $J_{\mbox{\tiny CE}}(\theta)$ and $J_{\mbox{\tiny HARD}}(\theta)$ conditioned on $x$.
  The difference $J_{\mbox{\tiny CE}}^x(\theta) - J_{\mbox{\tiny HARD}}^x(\theta)$ is
  \begin{align}
    \log Z_\theta(x) - \expected{\substack{y_1 \sim \pop(\cdot|x) \\ y_{2:K} \sim h(\cdot|x,y_1)}}{ \log Z_\theta(x,y_{1:K})} \label{eq:proof1}
  \end{align}
  where we define $Z_\theta(x) = \sum_{y' \in \mathcal{Y}} \myexp{s_\theta(x,y')}$ and $Z_\theta(x,y_{1:K}) = \sum_{k=1}^K \exp(s_\theta(x,y_k))$.
  For any $(\tilde{x}, \tilde{y})$, the partial derivative of \eqref{eq:proof1} with respect to $s_\theta(\tilde{x},\tilde{y})$ is given by
  $\doubbr{x=\tilde{x}} p_\theta(\tilde{y}|x) - \doubbr{x = \tilde{x}}\ga_\theta(\tilde{y}|x)$
  where $\doubbr{A}$ is the indicator function that takes the value 1 if $A$ is true and $0$ otherwise.
  Taking an expectation of their difference over $x \sim \pop$ gives the partial derivative of $b(\theta) = J_{\mbox{\tiny CE}}(\theta)-J_{\mbox{\tiny HARD}}(\theta)$ with respect to $s_\theta(\tilde{x},\tilde{y})$:
  $\pop(\tilde{x}) (p_\theta(\tilde{y}|\tilde{x}) - \ga_\theta(\tilde{y}|\tilde{x}))$.
  The statement follows from the chain rule:
  \begin{align*}
    b_i(\theta) = \sum_{\substack{x \in \mathcal{X}, y \in \mathcal{Y}}} \frac{\partial b(\theta)}{\partial s_\theta(x,y)} \frac{\partial s_\theta(x,y)}{\partial \theta_i}
  \end{align*}
\end{proof}

Theorem~\ref{thm:bias} states that the bias vanishes if $\ga_\theta(y|x)$ matches $p_\theta(y|x)$.
Hard negative mining can be seen as an attempt to minimize the bias by defining $h(\cdot|x,y_1)$ in terms of $p_\theta$.
Specifically, we define
\begin{align}
  &h(y_{2:K}|x,y_1) \notag \\
  &\propto \doubbr{\abs{\myset{y_1 \ldots y_K}} = K} \prod_{k=2}^K  p_\theta(y_k|x) \label{eq:mine}
\end{align}
Thus $h(\cdot|x,y_1)$ has support only on $y_{2:K} \in \mathcal{Y}^{K-1}$ that are distinct and do not contain the gold.
Greedy sampling from $h(\cdot|x,y_1)$ corresponds to taking $K-1$ incorrect label types with highest scores.
This coincides with the hard negative mining strategy of \citet{gillick-etal-2019-learning}.

The absence of duplicates in $y_{1:K}$ ensures $J_{\mbox{\tiny CE}}(\theta) = J_{\mbox{\tiny HARD}}(\theta)$ if $K=\abs{\mathcal{Y}}$.
This is consistent with (but does not imply) Theorem~\ref{thm:bias} since in this case $\ga_\theta(y|x) = p_\theta(y|x)$.
For general $K < \abs{\mathcal{Y}}$, Theorem~\ref{thm:bias} still gives a precise bias term.
To gain a better insight into its behavior, it is helpful to consider a heuristic approximation given by\footnote{We can rewrite $\ga_\theta(y|x)$ as
\begin{align*}
    \expected{\substack{y_1 \sim \pop(\cdot|x) \\ y_{2:K} \sim h(\cdot|x,y_1)}}{   \frac{ \mathrm{count}_{y_{1:K}}(y) \myexp{s_\theta(x,y)}}{ \sum_{y' \in \mathcal{Y}}  \mathrm{count}_{y_{1:K}}(y') \myexp{s_\theta(x,y')} } }
\end{align*}
where $\mathrm{count}_{y_{1:K}}(y)$ is the number of times $y$ appears in $y_{1:K}$.
The approximation uses $\mathrm{count}_{y_{1:K}}(y) \approx p_\theta(y|x)$ under \eqref{eq:mine}.
%If $K$ is large, $y$ appears in $y_{1:K}$ with probability approximately $p_\theta(y|x)$ since $y_k \sim p_\theta(\cdot|x)$ for $k > 1$.
}
\begin{align*}
  \ga_\theta(y|x) \approx \frac{p_\theta(y|x) \myexp{s_\theta(x,y)}}{ N_\theta(x) }
\end{align*}
where $N_\theta(x) = \sum_{y' \in \mathcal{Y}} p_\theta(y'|x) \myexp{s_\theta(x,y')}$.
Plugging this approximation in Theorem~\ref{thm:bias} we have a simpler equation
\begin{align*}
  b_i(\theta) \approx \expected{(x,y) \sim \pop}{\paren{1 - \de_\theta(x, y) } \frac{\partial s_\theta(x,y)}{\partial \theta_i}}
\end{align*}
where $\de_\theta(x, y) = \myexp{s_\theta(x,y)} / N_\theta(x)$.
The expression suggests that the bias becomes smaller as the model improves since $p_\theta(\cdot|x) \approx \pop(\cdot|x)$ implies $\de_\theta(x, y) \approx 1$ where $(x,y) \sim \pop$.

We can formalize the heuristic argument to prove a desirable property of \eqref{eq:nce-hard}:
the gradient is unbiased if $\theta$ satisfies $p_\theta(y|x) = \pop(y|x)$, assuming iid hard negatives.

\begin{thm}
  Assume $K \geq 2$ and the distribution $h(y_{2:K}|x,y_1) = \prod_{k=2}^K p_\theta(y_k|x)$ in \eqref{eq:nce-hard}.
  If $p_\theta(y|x) = \pop(y|x)$, then
  $\nabla J_{\mbox{\tiny HARD}}(\theta) = \nabla J_{\mbox{\tiny CE}}(\theta)$.
  \label{thm:unbiased}
\end{thm}

\begin{proof}
  Since $\pop(y|x) = \exp(s_\theta(x,y))/Z_\theta(x)$,
  the probability $\ga_\theta(y|x)$ in \eqref{eq:gamma} is
  \begin{align*}
    &\sum_{y_{1:K} \in \mathcal{Y}^K} \prod_{k=1}^K \frac{\myexp{s_\theta(x,y_k)}}{Z_\theta(x)} \frac{\myexp{s_\theta(x,y)}}{Z_\theta(x,y_{1:K})} \\
    &= \frac{\myexp{s_\theta(x,y)}}{Z_\theta(x)} \sum_{y_{1:K} \in \mathcal{Y}^K} \frac{\prod_{k=1}^K  \myexp{s_\theta(x,y_k)}}{Z_\theta(x,y_{1:K})}
  \end{align*}
  The sum marginalizes a product distribution over $y_{1:K}$, thus equals one.
  Hence $\ga_\theta(y|x) = p_\theta(y|x)$. The statement follows from Theorem~\ref{thm:bias}.
\end{proof}

The proof exploits the fact that negative examples are drawn from the model and does not generally hold for other negative distributions (e.g., uniformly random).
We empirically verify that hard negatives indeed yield a drastically smaller bias compared to random negatives (Section~\ref{sec:bias}).

\subsection{Adversarial Learning}

We complement the bias-based view of hard negatives with an adversarial view.
We generalize \eqref{eq:nce-hard} and define
\begin{align*}
  J_{\mbox{\tiny ADV}}(\theta, h) = \expected{\substack{(x,y_1) \sim \pop \\ y_{2:K} \sim h(\cdot|x,y_1)}}{ -\log \pi_\theta(1|x,y_{1:K}) }
\end{align*}
where we additionally consider the choice of a hard-negative distribution.
The premise of adversarial learning is that it is beneficial for $\theta$ to consider the worst-case scenario when minimizing this loss.
This motivates a nested optimization problem:
\begin{align*}
  \min_{\theta \in \R^d}\; \max_{h \in \mathcal{H}}\; J_{\mbox{\tiny ADV}}(\theta, h)
\end{align*}
where $\mathcal{H}$ denotes the class of conditional distributions over $S \subset \mathcal{Y}$ satisfying $\abs{S \cup \myset{y_1}} = K$.

\begin{thm}
  Fix $\theta \in \R^d$. For any $(x,y_1)$, pick
  \begin{align*}
    \tilde{y}_{2:K} \in \argmax_{\substack{y_{2:K} \in \mathcal{Y}^{K-1}:\\ \abs{\myset{y_1 \ldots y_K}} = K}}\; \sum_{k=2}^K s_\theta(x,y_k)
  \end{align*}
  breaking ties arbitrarily, and define the point-mass distribution over $\mathcal{Y}^{K-1}$:
  \begin{align*}
    \tilde{h}(y_{2:K}|x,y_1) = \doubbr{ y_k = \tilde{y}_k\; \forall k = 2 \ldots K}
  \end{align*}
  Then $\tilde{h} \in \argmax_{h \in \mathcal{H}} J_{\mbox{\tiny ADV}}(\theta, h)$.
\label{thm:adv}
\end{thm}

\begin{proof}
  $\max_{h \in \mathcal{H}} J_{\mbox{\tiny ADV}}(\theta, h)$ is equivalent to
  \begin{align*}
    \max_{h \in \mathcal{H}} \expected{\substack{(x,y_1)\sim \pop \\ y_{2:K} \sim h(\cdot|x,y_1)}}{\log \sum_{k=1}^K \myexp{ s_\theta(x,y_k)}}
  \end{align*}
  The expression inside the expectation is maximized by $\tilde{y}_{2:K}$ by the monotonicity of $\log$ and $\exp$,
  subject to the constraint that $\abs{\myset{y_1 \ldots y_K}} = K$.
  $\tilde{h} \in \mathcal{H}$ achieves this maximum.
\end{proof}

\section{Score Function}
\label{sec:score}

Along with the choice of negatives, the choice of the score function $s_\theta:\mathcal{X} \times \mathcal{Y} \ra \R$ is a critical component of NCE in practice.
There is a clear trade-off between performance and efficiency in modeling the cross interaction between the input-label pair $(x,y)$. %, each of which is typically a long sequence in NLP.
This trade-off spurred many recent works to propose various architectures in search of a sweet spot \citep{poly,luan2020sparse},
but they are developed in isolation of one another and difficult to compare.
In this section, we give a general algebraic form of the score function that subsumes many of the existing works as special cases.

\subsection{General Form}

We focus on the standard setting in NLP in which $x \in \mathcal{V}^T$ and $y \in \mathcal{V}^{T'}$ are sequences of tokens in a vocabulary $\mathcal{V}$.
Let $E(x) \in \R^{H \by T}$ and $F(y) \in \R^{H \by T'}$ denote their encodings, typically obtained from the final layers of separate pretrained transformers like \bert \citep{devlin-etal-2019-bert}.
We follow the convention popularized by \bert and assume the first token is a special symbol (i.e., [CLS]), so that $E_1(x)$ and $F_1(y)$ represent single-vector summaries of $x$ and $y$.
We have the following design choices:
\begin{itemizesquish}
\item \textbf{Direction}: If $x \ra y$, define the query $Q = E(x)$ and key $K = F(y)$. If $y \ra x$, define the query $Q = F(y)$ and key $K = E(x)$.
\item \textbf{Reduction}: Given integers $m, m'$, reduce the number of columns in $Q$ and $K$ to obtain $Q_m \in \R^{H \by m}$ and $K_{m'} \in \R^{H \by m'}$.
  We can simply select leftmost columns, or introduce an additional layer to perform the reduction.
\item \textbf{Attention}: Choose a column-wise attention $\mathrm{Attn}: A \mapsto \wb{A}$ either $\mathrm{Soft}$ or $\mathrm{Hard}$.
  If $\mathrm{Soft}$, $\wb{A}_t = \mathrm{softmax}(A_t)$ where the subscript denotes the column index.
  If $\mathrm{Hard}$, $\wb{A}_t$ is a vector of zeros with exactly one $1$ at index $\argmax_i [A_t]_i$.
\end{itemizesquish}
Given the design choices, we define the score of $(x,y)$ as
\begin{align}
  s_\theta(x,y) = 1_m^\top Q_m^\top K_{m'} \mathrm{Attn}\paren{K_{m'}^\top Q_m} \label{eq:score}
\end{align}
where $1_m$ is a vector of $m$ 1s that aggregates query scores.
Note that the query embeddings $Q_m$ double as the value embeddings.
The parameter vector $\theta \in \R^d$ denotes the parameters of the encoders $E, F$ and the optional reduction layer.

\subsection{Examples}
\label{subsec:examples}

\paragraph{Dual encoder.} Choose either direction $x \ra y$ or $y \ra x$. Select the leftmost $m = m' = 1$ vectors in $Q$ and $K$ as the query and key.
The choice of attention has no effect.
This recovers the standard dual encoder used in many retrieval problems \citep{gupta-etal-2017-entity,lee2019latent,logeswaran2019zero,wu2020scalable,karpukhin2020dense,guu2020realm}: $s_\theta(x,y) = E_1(x)^\top F_1(y)$.

\paragraph{Poly-encoder.} Choose the direction $y \ra x$.
Select the leftmost $m=1$ vector in $F(y)$ as the query.
Choose an integer $m'$ and compute $K_{m'} = E(x) \mathrm{Soft}(E(x)^\top O)$ where
$O \in \R^{H \by m'}$ is a learnable parameter  (``code'' embeddings). Choose soft attention.
This recovers the poly-encoder \citep{poly}:
$s_\theta(x,y) = F_1(y)^\top C_{m'}(x, y)$
where $C_{m'}(x, y) = K_{m'} \mathrm{Soft}\paren{K_{m'}^\top F_1(y)}$.
Similar architectures without length reduction have been used in previous works,
for instance the neural attention model of \citet{ganea-hofmann-2017-deep}.

\paragraph{Sum-of-max.} Choose the direction $x \ra y$.
Select all $m=T$ and $m'=T'$ vectors in $E(x)$ and $F(y)$ as the query and key.
Choose $\mathrm{Attn} = \mathrm{Hard}$.
This recovers the sum-of-max encoder (aka., ColBERT) \citep{colbert}:
$s_\theta(x,y) = \sum_{t=1}^T \max_{t'=1}^{T'} E_t(x)^\top F_{t'}(y)$.

\paragraph{Multi-vector.} Choose the direction $x \ra y$.
Select the leftmost $m=1$ and $m'=8$ vectors in $E(x)$ and $F(y)$ as the query and key.
Choose $\mathrm{Attn} = \mathrm{Hard}$.
This recovers the multi-vector encoder \citep{luan2020sparse}:
$s_\theta(x,y) = \max_{t'=1}^{m'} E_1(x)^\top F_{t'}(y)$.
It reduces computation to fast dot products over cached embeddings, but is less expressive than the sum-of-max.

\vspace{3mm}
The abstraction \eqref{eq:score} is useful because it generates a spectrum of architectures as well as unifying existing ones.
For instance, it is natural to ask if we can further improve the poly-encoder by using $m > 1$ query vectors.
We explore these questions in experiments.

\section{Related Work}
\label{sec:related}

We discuss related work to better contextualize our contributions.
There is a body of work on developing unbiased estimators of the population distribution by modifying NCE.
The modifications include learning the normalization term as a model parameter \citep{gutmann2012noise,mnih2012fast}
and using a bias-corrected score function \citep{ma-collins-2018-noise}.
However, they assume unconditional negative distributions and do not explain the benefit of hard negatives in NCE
\citep{gillick-etal-2019-learning,wu2020scalable,karpukhin2020dense,wikipretrainthing}.
In contrast, we directly consider the hard-negative NCE loss used in practice \eqref{eq:nce-hard}, and
justify it as a biased estimator of the gradient of the cross-entropy loss.

Our work is closely related to prior works on estimating the gradient of the cross-entropy loss, again by modifying NCE.
They assume the following loss \citep{bengio2008adaptive}, which we will denote by $J_{\mbox{\tiny PRIOR}}(\theta)$:
\begin{align}
  \expected{\substack{(x,y_1) \sim \pop \\ y_{2:K} \sim \nu(\cdot|x,y_1)^K}}{ -\log \frac{\myexp{\bar{s}_\theta(x,y_1,y_1)}}{\sum_{k=1}^K \myexp{\bar{s}_\theta(x,y_1, y_k)}}} \label{eq:hard-prior}
\end{align}
Here, $\nu(\cdot|x,y_1)$ is a conditional distribution over $\mathcal{Y}\backslash \myset{y_1}$, and
$\bar{s}_\theta(x,y',y)$ is equal to $s_\theta(x,y)$ if $y = y'$ and
$s_\theta(x,y) - \log((K-1)\nu(y|x,y_1))$ otherwise.
It can be shown that $\nabla J_{\mbox{\tiny PRIOR}}(\theta) = \nabla J_{\mbox{\tiny CE}}(\theta)$ iff $\nu(y|x,y_1) \propto \exp(s_\theta(x,y))$
for all $y \in \mathcal{Y}\backslash \myset{y_1}$ \citep{blanc2018adaptive}.
However, \eqref{eq:hard-prior} requires adjusting the score function and iid negative examples, thus less aligned with practice than \eqref{eq:nce-hard}.
The bias analysis of $\nabla J_{\mbox{\tiny PRIOR}}(\theta)$ for general $\nu(\cdot|x,y_1)$ is also significantly more complicated than Theorem~\ref{thm:bias} \citep{rawat2019sampled}.

There is a great deal of recent work on unsupervised contrastive learning of image embeddings in computer vision \citep[\textit{inter alia}]{oord2018representation,hjelm2018learning,chen2020simple}.
Here, $s_\theta(x,y) = E_\theta(x)^\top F_\theta(y)$ is a similarity score between images, and $E_\theta$ or $F_\theta$ is used to produce useful image representations for downstream tasks.
The model is again learned by \eqref{eq:nce} where $(x, y_1)$ are two random corruptions  of the same image and $y_{2:K}$ are different images.
\citet{robinson2021contrastive} propose a hard negative distribution in this setting and analyze the behavior of learned embeddings under that distribution.
In contrast, our setting is large-scale supervised classification, such as entity linking, and our analysis is concerned with NCE with general hard negative distributions.

In a recent work, \citet{xiong2021approximate} consider contrastive learning for text retrieval with hard negatives
obtained globally from the whole data with asynchronous updates, as we do in our experiments.
They use the framework of importance sampling to argue that hard negatives yield gradients with larger norm, thus smaller variance and faster convergence.
However, their argument does not imply our theorems. They also assume a pairwise loss, excluding non-pairwise losses such as \eqref{eq:nce}.

\section{Experiments}
\label{sec:experiments}

We now study empirical aspects of the hard-negative NCE (Section~\ref{sec:hard}) and the spectrum of score functions (Section~\ref{sec:score}).
Our main testbed is Zeshel \citep{logeswaran2019zero}, a challenging dataset for zero-shot entity linking.
We also present complementary experiments on AIDA CoNLL-YAGO \citep{hoffart2011robust}.\footnote{Our code is available at: \url{https://github.com/WenzhengZhang/hard-nce-el}.}

\subsection{Task}

Zeshel contains 16 domains (fictional worlds like \textit{Star Wars}) partitioned to 8 training and 4 validation and test domains.
Each domain has tens of thousands of entities along with their textual descriptions, which contain references to other entities in the domain and double as labeled mentions.
The input $x$ is a contextual mention and the label $y$ is the description of the referenced entity.
A score function $s_\theta(x,y)$ is learned in the training domains and applied to a new domain for classification and retrieval.
Thus the model must read descriptions of unseen entities and still make correct predictions.

We follow prior works and report micro-averaged top-64 recall and macro-averaged accuracy for evaluation.
The original Zeshel paper \citep{logeswaran2019zero} distinguishes normalized vs unnormalized accuracy.
Normalized accuracy assumes the presence of an external retriever and considers a mention only if its gold entity is included in top-64 candidates from the retriever.
In this case, the problem is reduced to reranking and a computationally expensive joint encoder can be used.
Unnormalized accuracy considers all mentions. %and is upper bounded by the poor recall of BM25 (first row of Table~\ref{tab:recall}).
Our goal is to improve unnormalized accuracy.

\citet{logeswaran2019zero} use BM25 for retrieval, which upper bounds unnormalized accuracy by its poor recall (first row of Table~\ref{tab:recall}).
\citet{wu2020scalable} propose a two-stage approach in which a dual encoder is trained by hard-negative NCE and held fixed,
then a \bert-based joint encoder is trained to rerank the candidates retrieved by the dual encoder.
This approach gives considerable improvement in unnormalized accuracy, primarily due to the better recall of a trained dual encoder over BM25 (second row of Table~\ref{tab:recall}).
We show that we can further push the recall by optimizing the choice of hard negatives and architectures.

\subsection{Architectures}

We represent $x$ and $y$ as length-$128$ wordpiece sequences where the leftmost token is the special symbol [CLS];
we mark the boundaries of a mention span in $x$ with special symbols.
We use two independent \bert-bases to calculate mention embeddings $E(x) \in \R^{768 \by 128}$ and entity embeddings $F(y) \in \R^{768 \by 128}$,
where the columns $E_t(x), F_t(y)$ are contextual embeddings of the $t$-th tokens.

\paragraph{Retriever.}
The retriever defines $s_\theta(x, y)$, the score between a mention $x$ and an entity $y$, by one of the architectures described in Section~\ref{subsec:examples}:
\begin{align*}
  E_1(x)^\top F_1(y) &&\mbox{ (\dual)} \\
  F_1(y)^\top C_m(x, y) &&\mbox{ (\poly-$m$)} \\
  \mbox{$\max_{t=1}^m$} E_1(x)^\top F_t(y) &&\mbox{ (\me-$m$)} \\
  \mbox{$\sum_{t=1}^{128} \max_{t'=1}^{128}$} E_t(x)^\top F_{t'}(y) &&\mbox{ (\som)}
\end{align*}
denoting the dual encoder, the poly-encoder \citep{poly},
the multi-vector encoder \citep{luan2020sparse}, and the sum-of-max encoder \citep{colbert}.
These architectures are sufficiently efficient to calculate $s_\theta(x,y)$ for all entities $y$ in training domains for each mention $x$.
This efficiency is necessary for sampling hard negatives during training and retrieving candidates at test time.

\paragraph{Reranker.}
The reranker defines $s_\theta(x,y) = w^\top E_1(x,y) + b$
%We define a full joint encoder $s_\theta(x,y) = w^\top E_1(x,y) + b$ for reranking only
where $E(x,y) \in \R^{H \by 256}$ is \bert (either base $H = 768$ or large $H=1024$) embeddings of the concatenation of $x$ and $y$ separated by the special symbol [SEP],
and $w, b$ are parameters of a linear layer.
We denote this encoder by \joint.

\subsection{Optimization}
\label{sec:opt}

\paragraph{Training a retriever.}
A retriever is trained by minimizing an empirical estimate of the hard-negative NCE loss \eqref{eq:nce-hard},
\begin{align}
  \wh{J}_{\mbox{\tiny HARD}}(\theta) = - \frac{1}{N} \sum_{i=1}^N  \log \frac{\myexp{s_\theta(x_i,y_{i,1})}}{ \sum_{k'=1}^K \myexp{s_\theta(x_i,y_{i,k'})} }
  \label{eq:emp-loss}
\end{align}
where $(x_1, y_{1,1}) \ldots (x_N, y_{N,1})$ denote $N$ mention-entity pairs in training data, and
$y_{i,2} \ldots y_{i,K} \sim h(\cdot|x_i, y_{i,1})$ are $K-1$ negative entities for the $i$-th mention.
We vary the choice of negatives as follows.
\begin{itemizesquish}
\item Random: The negatives are sampled uniformly at random from all entities in training data.
\item Hard: The negatives are sampled from \eqref{eq:mine} each epoch.
  That is, in the beginning of each training pass, for each $i$ we sample entities $y_{i,2} \ldots y_{i,K}$
  from $\mathcal{Y} \backslash \myset{y_{i,1}}$  without replacement with probabilities proportional to $\myexp{s_\theta(x_i,y_{i,k})}$.
  %distinct highest-scoring incorrect entities from all entities in training data for each mention under the current model $s_\theta$.
  This is slightly different from, and simpler than, the original hard negative mining strategy of \citet{gillick-etal-2019-learning}
  which pretrains the model using random negatives then greedily adds negative entities that score higher than the gold.
\item Mixed-$p$: $p$ percent of the negatives are hard, the rest are random.
  Previous works have shown that such a combination of random and hard negatives can be effective.
  We find the performance is not sensitive to the value of $p$ (Appendix~\ref{app:p-values}).
\end{itemizesquish}
We experimented with in-batch sampling as done in previous works (e.g., \citet{gillick-etal-2019-learning}),
but found sampling from all training data to be as effective and more straightforward (e.g., the number of random negatives is explicitly unrelated to the batch size).
We use $K=64$ in all experiments.

\paragraph{Training a reranker.}
We use \joint only for reranking by minimizing \eqref{eq:emp-loss} with top-63 negatives given by a fixed retriever,
where we vary the choice of retriever.
We also investigate other architectures for reranking such as the poly-encoder and the sum-of-max encoder,
but we find the full cross attention of \joint to be indispensable.
Details of reranking experiments can be found in Appendix~\ref{app:reranking}.

\begin{figure}[t!]
\vskip 0.2in
\begin{center}
  \centerline{
    \includegraphics[width=\columnwidth/2]{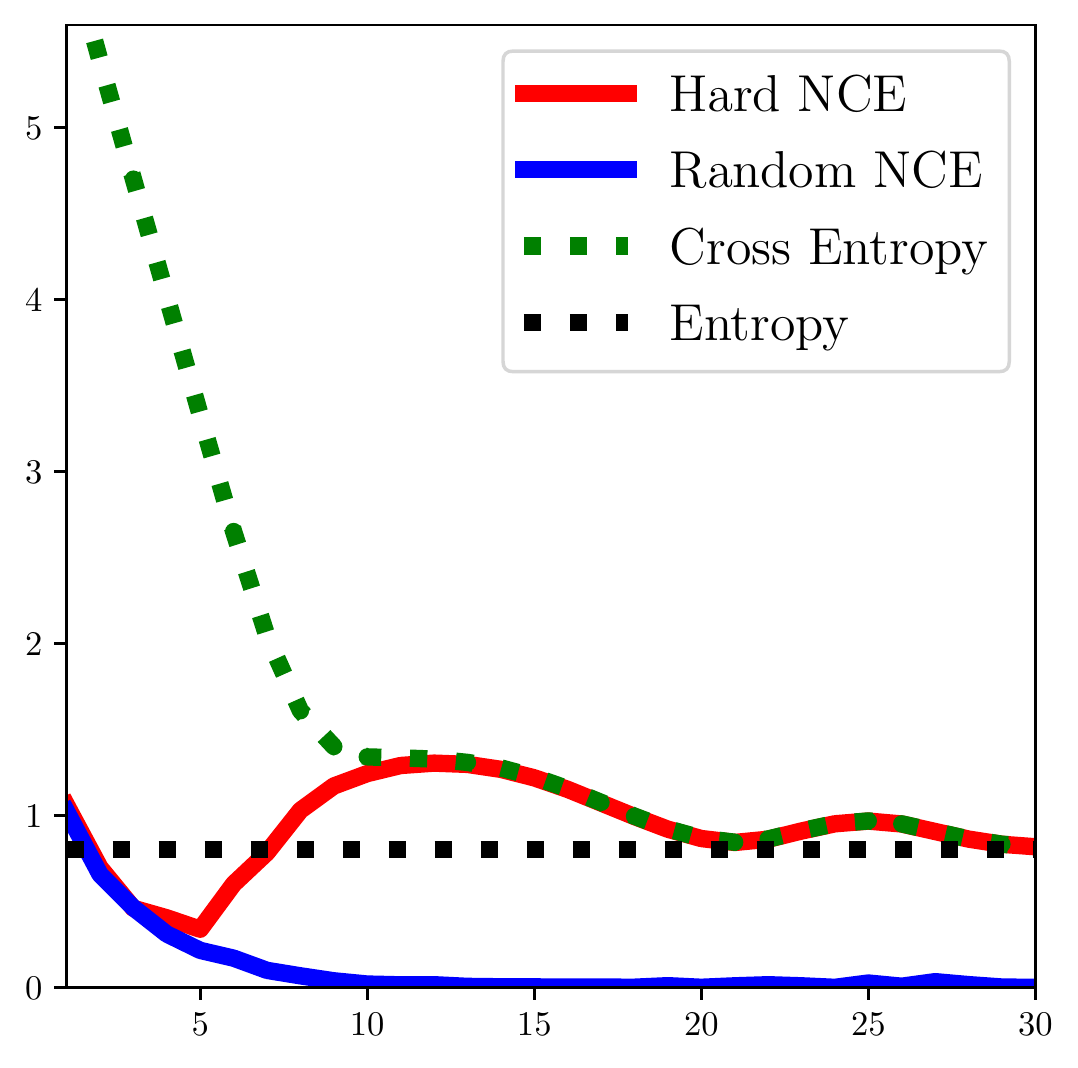}
    \includegraphics[width=\columnwidth/2]{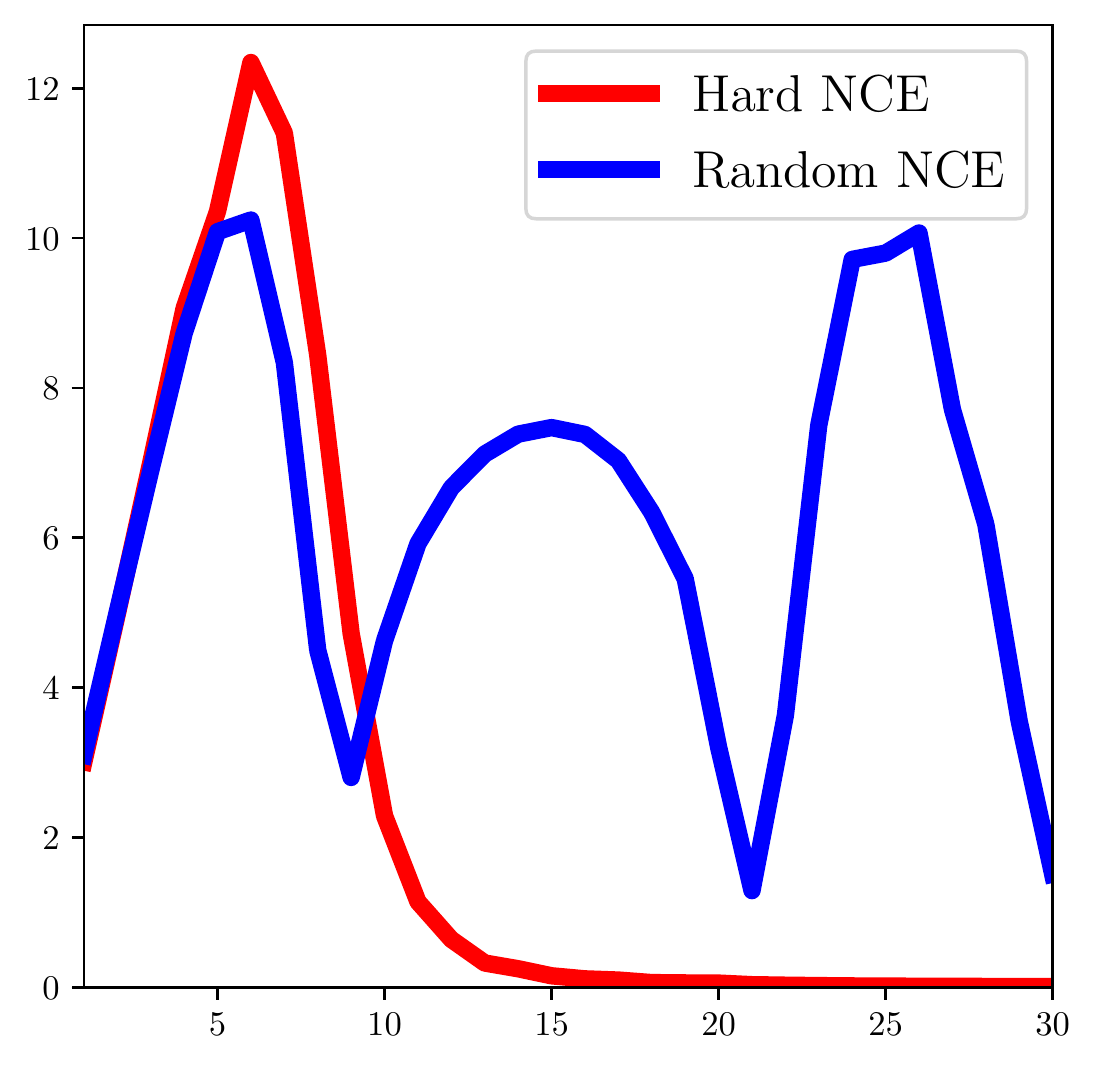}
  }
  \caption{Synthetic experiments.
  We use a feedforward network to estimate the population distribution by minimizing sampled cross entropy in each step ($x$-axis).
  We show the NCE loss (left) and the norm of the gradient bias (right) using hard vs random negatives.}
\label{fig:synthetic}
\end{center}
\end{figure}

\paragraph{Other details.}
All models are trained up to 4 epochs using Adam.
We tune the learning rate over $\myset{5\mathrm{e}{-5},2\mathrm{e}{-5},1\mathrm{e}{-5}}$ on validation data.
We use the training batch size of $4$ mentions for all models except for \joint, for which we use $2$.
%The hard negatives are recomputed for all training data in batches in the beginning of each epoch.
Training time is roughly half a day on a single NVIDIA A100 GPU for all models, except the \som retriever which takes 1-2 days.

\begin{table}[t!]
  {\small
  \centering
\begin{tabular}{lccc}
\hline
\textbf{Model} & \textbf{Negatives} & \textbf{Val} & \textbf{Test}\\
\hline
\bm & -- &  76.22 & 69.13 \\
\citet{wu2020scalable} & Mixed (10 hard) &  91.44 & 82.06 \\
\hline
\hline
\dual   & Random & 91.08 & 81.80 \\
& Hard    & 91.99 & 84.87 \\
& Mixed-50 & 91.75 & 84.16 \\
\dual-{\scriptsize \eqref{eq:hard-prior}}   & Hard & 91.57 & 83.08 \\
\hline
\poly-{\scriptsize 16}  & Random & 91.05 & 81.73 \\
& Hard    & 92.08 & 84.07 \\
& Mixed-50 & 92.18 & 84.34 \\
%\poly-{\scriptsize 64}  & Hard & 92.16 & 84.80 \\
\hline
\me-{\scriptsize 8} & Random & 91.13 &  82.44 \\
& Hard & 92.35 & 84.94 \\
& Mixed-50 & 92.76 & 84.11 \\
\hline
\som & Random & 92.51 &  87.62 \\
& Hard & 94.49 & 88.68 \\
& Mixed-50 & \textbf{94.66} & \textbf{89.62} \\
\hline
\end{tabular}
\caption{Top-64 recalls over different choices of architecture and negative examples for a retriever trained by NCE.
  \citet{wu2020scalable} train a dual encoder by NCE with 10 hard negatives.
  \dual-{\small \eqref{eq:hard-prior}} is \dual trained with the score-adjusted loss \eqref{eq:hard-prior}.
}
\label{tab:recall}
}
\end{table}

\begin{table*}
  {%\small
  \centering
\begin{tabular}{lccccc}
\hline
\textbf{Model}                 &            \textbf{Retriever} &   \textbf{Negatives}  &   \textbf{Joint Reranker}  &     \multicolumn{2}{c}{\textbf{Unnormalized}}  \\
                               &                               &                       &                      &     \textbf{Val} & \textbf{Test}  \\
\hline
\citet{logeswaran2019zero}     &       BM25                    &    --                 & base                 &  --          & 55.08 \\
\citet{logeswaran2019zero}+DAP &       BM25                    &    --                 & base                 &  --          & 55.88 \\
\citet{wu2020scalable}         &       \dual (base)            &    Mixed (10 hard)    & base                 &  --          & 61.34 \\
\citet{wu2020scalable}         &       \dual (base)            &    Mixed (10 hard)    & large                &  --          & 63.03 \\
\hline
\hline
Ours                           &       \dual (base)            &    Hard               & base                 & 69.14        & 65.42 \\
                               &       \dual (base)            &    Hard               & large                & 68.31        & 65.32 \\
                               &       \som  (base)            &    Hard               & base                 & 69.19        & 66.67 \\
                               &       \som  (base)            &    Hard               & large                & 70.08        & 65.95 \\
                               &       \som  (base)            &    Mixed-50           & base                 & 69.22        & 65.37 \\
                               &       \som  (base)            &    Mixed-50           & large                & \textbf{70.28}        & \textbf{67.14} \\
\hline
\end{tabular}
\caption{Unnormalized accuracies with two-stage training. DAP refers to domain adaptive pre-training on source and target domains.
}
\label{tab:main}
}
\end{table*}

\begin{table*}[!ht]
{\scriptsize
\centering
\begin{tabular}{l|l}
\hline
Mention & $\ldots$ his temporary usurpation of the Imperial throne by invading and seized control of the Battlespire, the purpose of this being to cripple \\
& the capacity of the Imperial College of Battlemages, which presented a threat to Tharn's power as Emperor. Mehrunes Dagon was \\
& responsible for the \textbf{destruction} of Mournhold at the end of the First Era, and apparently also $\ldots$ \\
\hline
Random
&   1. \textbf{Mehrunes Dagon} is one of the seventeen Daedric Princes of Oblivion and the primary antagonist of $\ldots$ \\
&   2. \textbf{Daedric Forces of Destruction} were Mehrunes Dagon's personal army, hailing from his realm of Oblivion, the Deadlands. $\ldots$\\
&   3. \textbf{Weir Gate} is a device used to travel to Battlespire from Tamriel. During the Invasion of the Battlespire, Mehrunes Dagon's forces $\ldots$ \\
&   4. \textbf{Jagar Tharn} was an Imperial Battlemage and personal adviser to Emperor Uriel Septim VII. Tharn used the Staff of Chaos $\ldots$ \\
&   5. \textbf{House Sotha} was one of the minor Houses of Vvardenfell until its destruction by Mehrunes Dagon in the times of Indoril Nerevar. $\ldots$ \\
&   6. \textbf{Imperial Battlespire} was an academy for training of the Battlemages of the Imperial Legion. The Battlespire was moored in $\ldots$ \\
\hline
Hard
& 1. \textbf{Fall of Ald'ruhn} was a battle during the Oblivion Crisis. It is one of the winning battles invading in the name of Mehrunes Dagon $\ldots$ \\
& 2. \textbf{Daedric Forces of Destruction} were Mehrunes Dagon's personal army, hailing from his realm of Oblivion, the Deadlands. $\ldots$\\
& 3. \textbf{House Sotha} was one of the minor Houses of Vvardenfell until its destruction by Mehrunes Dagon in the times of Indoril Nerevar. $\ldots$ \\
& \checkmark 4. \textbf{Sack of Mournhold} was an event that occurred during the First Era. It was caused by the Dunmer witch Turala Skeffington $\ldots$ \\
& 5. \textbf{Mehrunes Dagon of the House of Troubles} is a Tribunal Temple quest, available to the Nerevarine in $\ldots$ \\
& 6. \textbf{Oblivion Crisis}, also known as the Great Anguish to the Altmer or the Time of Gates by Mankar Camoran, was a period of major turmoil $\ldots$ \\
\hline
\end{tabular}
\caption{A retrieval example with hard negative training on Zeshel. We use a \som retriever trained with random vs hard negatives (92.51 vs 94.66 in top-64 validation recall).
We show a validation mention (\textbf{destruction}) whose gold entity is retrieved by the hard-negative model but not by the random-negative model.
Top entities are shown for each model (title boldfaced); the correct entity is \textbf{Sack of Mournhold} (checkmarked).
}
\label{tab:example}
}
\end{table*}

\subsection{Bias}
\label{sec:bias}

We conduct experiments on synthetic data to empirically validate our bias analysis in Section~\ref{sec:grad}.
Analogous experiments on Zeshel with similar findings can be found in Appendix~\ref{app:bias-zeshel}.

We construct a population distribution over 1000 labels with small entropy to represent the peaky conditional label distribution $\pop(y|x)$.
We use a feedforward network with one ReLU layer to estimate this distribution by minimizing the empirical cross-entropy loss based on 128 iid samples per update.
At each update, we compute cross-entropy \eqref{eq:ce} exactly, and estimate NCE \eqref{eq:nce-hard} with 4 negative samples by Monte Carlo (10 simulations).

Figure~\ref{fig:synthetic} plots the value of the loss function (left) and the norm of the gradient bias (right) across updates.
We first observe that hard NCE yields an accurate estimate of cross entropy even with 4 samples.
In contrast, random NCE quickly converges to zero, reflecting the fact that the model can trivially discriminate between the gold and random labels.
We next observe that the bias of the gradient of hard NCE vanishes as the model distribution converges to the population distribution,
which supports our analysis that the bias becomes smaller as the model improves.
The bias remains nonzero for random NCE.

\subsection{Retrieval}

Table~\ref{tab:recall} shows the top-64 recall (i.e., the percentage of mentions whose gold entity is included in
the 64 entities with highest scores under a retriever trained by \eqref{eq:nce-hard})
as we vary architectures and negative examples.
We observe that hard and mixed negative examples always yield sizable improvements over random negatives, for all architectures.
Our dual encoder substantially outperforms the previous dual encoder recall by \citet{wu2020scalable},
likely due to better optimization such as global vs in-batch random negatives and the proportion of hard negatives.
We also train a dual encoder with the bias-corrected loss \eqref{eq:hard-prior} and find that this does not improve recall.
The poly-encoder and the multi-vector models are comparable to but do not improve over the dual encoder.
However, the sum-of-max encoder delivers a decisive improvement, especially with hard negatives, pushing the test recall to above 89\%.
Based on this finding, we use \dual and \som for retrieval in later experiments.

\subsection{Results}

We show our main results in Table~\ref{tab:main}.
Following \citet{wu2020scalable}, we do two-stage training in which we train a \dual or \som retriever with hard-negative NCE and
train a \joint reranker to rerank its top-64 candidates.
All our models outperform the previous best accuracy of 63.03\% by \citet{wu2020scalable}.
In fact, our dual encoder retriever using a \bert-base reranker outperforms the dual encoder retriever using a \bert-large reranker (65.42\% vs 63.03\%).
We obtain a clear improvement by switching the retriever from dual encoder to sum-of-max due to its high recall (Table~\ref{tab:recall}).
Using a sum-of-max retriever trained with mixed negatives and a \bert-large reranker gives the best result 67.14\%.

\subsection{Qualitative Analysis}

To better understand practical implications of hard negative mining,
we compare a \som retriever trained on Zeshel with random vs hard negatives (92.51 vs 94.66 in top-64 validation recall).
The mention categories most frequently improved are Low Overlap (174 mentions) and Multiple Categories (76 mentions) (see \citet{logeswaran2019zero} for the definition of these categories),
indicating that hard negative mining makes the model less reliant on string matching.
A typical example of improvement is shown in Table~\ref{tab:example}.
The random-negative model retrieves person, device, or institution entities because they have more string overlap (e.g. ``Mehrunes Dagon'', ``Battlespire'', and ``Tharn'').
In contrast, the hard-negative model appears to better understand that the mention is referring to a chaotic event like the Fall of Ald'ruhn, Sack of Mournhold, and Oblivion Crisis
and rely less on string matching.
We hypothesize that this happens because string matching is sufficient to make a correct prediction during training if negative examples are random,
but insufficient when they are hard.

To examine the effect of encoder architecture,
we also compare a \dual vs \som retriever both trained with mixed negatives (91.75 vs 94.66 in top-64 validation recall).
The mention categories most frequently improved are again Low Overlap (335 mentions) and Multiple Categories (41 mentions).
This indicates that cross attention likewise helps the model less dependent on simple string matching,
presumably by allowing for a more expressive class of score functions.

\begin{table}
  {%\small
  \centering
\begin{tabular}{l|ccc}
\hline
Model & Accuracy \\
\hline
\blink without finetuning & 80.27 \\
\blink with finetuning & 81.54 \\
\hline
\hline
\dual with $p=0$ & 82.40 \\
\dual with $p=50$  & 88.01 \\
\me-$2$ with $p=50$       & 88.39 \\
\me-$3$ with $p=50$       & 87.94 \\
\hline
\end{tabular}
\caption{Test accuracies on AIDA CoNLL-YAGO.
\blink refers to the two-stage model of \citet{wu2020scalable} pretrained on Wikipedia.
All our models are initialized from the \blink dual encoder and finetuned using
all 5.9 million Wikipedia entities as candidates.
}
\label{tab:aida}
}
\end{table}

\subsection{Results on AIDA}

We complement our results on Zeshel with additional experiments on AIDA.
We use \blink, a Wikipedia-pretrained two-stage model (a dual encoder retriever pipelined with a joint reranker, both based on \bert) made available by \citet{wu2020scalable}.\footnote{\url{https://github.com/facebookresearch/BLINK}}
We extract the dual encoder module from \blink and finetune it on AIDA using the training portion.
During finetuning, we use all 5.9 million Wikipedia entities as candidates to be consistent with prior work.
Because of the large scale of the knowledge base we do not consider \som and focus on the \me-$m$ retriever (\dual is a special case with $m=1$).
At test time, all models consider all Wikipedia entities as candidates.
For both AIDA and the Wikipedia dump, we use the version prepared by the KILT benchmark \citep{petroni2020kilt}.

Table~\ref{tab:aida} shows the results.
Since \citet{wu2020scalable} do not report AIDA results, we take the performance of \blink without and with finetuning from their GitHub repository
and the KILT leaderboard.\footnote{\url{https://ai.facebook.com/tools/kilt/} (as of April 8, 2021)}
We obtain substantially higher accuracy by mixed-negative training even without reranking.\footnote{We find that reranking does not improve accuracy on this task,
likely because the task does not require as much reading comprehension as Zeshel.}
There is no significant improvement from using $m > 1$ in the multi-vector encoder on this task.

\section{Conclusions}
\label{sec:conclusions}

Hard negatives can often improve NCE in practice, substantially so for entity linking \citep{gillick-etal-2019-learning},
but are used without justification.
We have formalized the role of hard negatives in quantifying the bias of the gradient of the contrastive loss with respect to the gradient of the full cross-entropy loss.
By jointly optimizing the choice of hard negatives and architectures, we have obtained new state-of-the-art results
on the challenging Zeshel dataset \citep{logeswaran2019zero}.

\section*{Acknowledgements}

This work was supported by the Google Faculty Research Awards Program.
We thank Ledell Wu for many clarifications on the BLINK paper.

%\bibliography{anthology,custom}  % KS: anthology takes too long to compile, just copy to custom
\bibliography{custom}
\bibliographystyle{acl_natbib}

\appendix

\section{Percentage of Hard Negatives}
\label{app:p-values}

We show top-64 validation recalls with varying values of the hard negative percentage $p$ in training below:
\begin{center}
  \begin{tabular}{l|ccc}
    \hline
    Mixed-$p$ (\%) & \dual  & \me-{\small 8} & \som \\
    \hline
    0 (Random)  & 91.08 & 91.13 & 92.51 \\
    25   & 92.18 & 92.74 & 94.13 \\
    50   & 91.75 & 92.76 & 94.66 \\
    75   & 92.24 & 93.41 & 94.37 \\
    100 (Hard)  & 92.05 & 93.27 & 94.54  \\
    \hline
  \end{tabular}
\end{center}
The presence of hard negatives is clearly helpful, but the exact choice of $p > 0$ is not as important.
We choose $p=50$ because we find that the presence of some random negatives often gives slight yet consistent improvement.

\section{Reranking Experiments}
\label{app:reranking}

We show the normalized and unnormalized accuracy of a reranker as we change the architecture while holding the retriever fixed:
\begin{center}{\small
    \begin{tabular}{lcccc}
      \hline
      \textbf{Model}            & \multicolumn{2}{c}{\textbf{Normalized}} & \multicolumn{2}{c}{\textbf{Unnormalized}}  \\
      & \textbf{Val} & \textbf{Test}            & \textbf{Val} & \textbf{Test}  \\
      \hline
      \dual                     &  60.43       &  62.49                   & 54.87        &   54.73       \\
      \poly-{\scriptsize 16}    &  60.37       &  60.98                   & 54.82        &   53.37       \\
      \poly-{\scriptsize 64}    &  60.80       &  61.88                   & 55.20        &   54.15       \\
      \poly-{\scriptsize 128}   &  60.60       &  62.72                   & 55.03        &   54.92       \\
      \me-{\scriptsize 8}       &  61.56       &  62.65                   & 55.90        &   54.87       \\
      \me-{\scriptsize 64}      &  61.94       &  62.94                   & 56.23        &   55.15       \\
      \me-{\scriptsize 128}     &  61.67       &  62.95                   & 55.98        &   55.17       \\
      \som                      &  65.38       &  65.24                   & 59.35        &   57.04       \\
      \genpoly-{\scriptsize 128}&  65.89       &  64.98                   & 59.82        &   56.82       \\
      \joint                    &  \textbf{76.17}       &  \textbf{74.90}                   & \textbf{69.14}        &   \textbf{65.42}       \\
      \hline
      \citeauthor{logeswaran2019zero}  &  76.06       &  75.06          & --           &   55.08       \\
      \citeauthor{wu2020scalable}      &  78.24       &  76.58          & --           &   --          \\
      \joint (ours)                    &  78.82       &  77.09          & 58.77        &   56.56       \\
      \hline
    \end{tabular}
}\end{center}
\genpoly-$m$ denotes a generalized version of the poly-encoder in which we use $m$ leftmost entity embeddings rather than one:
$s_\theta(x,y) = 1^\top_m F_{1:m}(y)^\top C_m(x, y)$.
We use a trained dual encoder with 91.93\% and 83.48\% validation/test recalls as a fixed retriever.
The accuracy increases with the complexity of the reranker.
The dual encoder and the poly-encoder are comparable, but the multi-vector, the sum-of-max, and the generalized poly-encoder
achieve substantially higher accuracies.
Not surprisingly, the joint encoder achieves the best performance.
We additionally show reranking results using the BM25 candidates provided in the Zeshel dataset
for comparison with existing results. %\citep{logeswaran2019zero,wu2020scalable}.
Our implementation of \joint with \bert-base obtains comparable accuracies.

\section{Bias Experiments on Zeshel}
\label{app:bias-zeshel}

%We take a random subset of the training data and estimate the difference between the gradient of the cross-entropy loss $J_{\mbox{\tiny CE}}(\theta)$ and the gradient of the hard-negative NCE loss $J_{\mbox{\tiny HARD}}(\theta)$ at different $\theta$s.

We consider the dual encoder $s_\theta(x,y) = E_1(x)^\top F_1(y)$ where $E$ and $F$ are parameterized by \bert-bases.
We randomly sample $64$ mentions, yielding a total of $128$ entities: $64$ referenced by the mentions, and $64$ whose descriptions contain these mentions.
We consider these $128$ entities to constitute the entirety of the label space $\mathcal{Y}$.
On the 64 mentions, we estimate $J_{\mbox{\tiny CE}}(\theta)$ by normalizing over the 128 entities;
we estimate $J_{\mbox{\tiny HARD}}(\theta)$ by normalizing over $K=8$ candidates where 7 are drawn from a negative distribution: either random, hard, or mixed.
Instead of a single-sample estimate as in \eqref{eq:emp-loss}, we draw negative examples 500 times and average the result.
We estimate the bias $b(\theta) \in \R^d$ by taking a difference between these two estimates
and report the norm below:
\begin{center}
  \begin{tabular}{lrr}
    \hline
    \textbf{Negatives} & $\|b(\theta_{\mbox{\tiny CE}})\|$ & $\|b(\theta_{\mbox{\tiny RAND}})\|$ \\
    \hline
    Random   & 16.33 & 166.38 \\
    Hard     &  0.68 &   0.09 \\
    Mixed-50 &  1.20 &   0.90 \\
    \hline
  \end{tabular}
\end{center}
We consider two parameter locations.
$\theta_{\mbox{\tiny CE}}$ is obtained by minimizing the cross-entropy loss (92.19\% accuracy).
$\theta_{\mbox{\tiny RAND}}$ is obtained by NCE with random negatives (60\% accuracy).
The bias is drastically smaller when negative examples are drawn from the model instead of randomly.
Mixed negatives yield comparably small biases.
With random negatives, the bias is much larger at $\theta_{\mbox{\tiny RAND}}$ since $\nabla J_{\mbox{\tiny CE}}(\theta_{\mbox{\tiny RAND}})$ is large.
In contrast, hard and mixed negatives again yield small biases.

\end{document}